\documentclass[11pt, oneside]{article}
\usepackage{geometry}
\usepackage{wrapfig}
\usepackage[utf8]{inputenc} 
\usepackage[T1]{fontenc}    
\usepackage{url}            
\usepackage{booktabs}       
\usepackage{amsfonts}       
\usepackage{nicefrac}       
\usepackage{microtype}      
\usepackage{xcolor}         
\usepackage{multicol, multirow}
\usepackage{bm}
\usepackage{enumitem}
\usepackage{siunitx}
\usepackage{makecell}

\usepackage{graphicx}	
\usepackage{amssymb}
\usepackage{amsfonts,latexsym,amsthm,amssymb,amsmath,amscd,euscript}
\usepackage{bbm}
\usepackage{framed}
\usepackage{mathrsfs}
\usepackage{natbib}

\usepackage{varioref}
\usepackage{hyperref}
\usepackage{xcolor}
\hypersetup{
    colorlinks,
    linkcolor={black!50!black},
    citecolor={green!50!black},
    urlcolor={green!80!black}
}

\usepackage[nameinlink,capitalize]{cleveref}

\theoremstyle{plain}
\newtheorem{theorem}{Theorem}
\newtheorem*{theorem*}{Theorem}
\newtheorem{lemma}[theorem]{Lemma}

\newtheorem{proposition}[theorem]{Proposition}

\newtheorem{claim}[theorem]{Claim}

\theoremstyle{definition}
\newtheorem{definition}{Definition}

\newtheorem{remark}[definition]{Remark}

\Crefname{claim}{Claim}{Claims}

\newcommand{\NN}{\mathbb{N}}
\newcommand{\RR}{\mathbb{R}}
\newcommand{\MX}{\mathcal{X}}
\newcommand{\MI}{\mathcal{I}}

\newcommand{\st}{\star}

\DeclareMathOperator*{\argmin}{arg\,min} 
\DeclareMathOperator*{\argmax}{arg\,max}
\newcommand{\disj}{\mathsf{DISJ}}

\newcommand{\multiset}[1]{\{\!\!\{{#1}\}\!\!\}}

\newcommand{\fsym}{f^{\mathsf{sym}}}
\newcommand{\tfsym}{\tilde f^{\mathsf{sym}}}

\usepackage{todonotes}

\newcommand{\ar}[1]{\todo[color=blue!35]{AR: #1}}
\newcommand{\dhruv}[1]{\todo{[dr: #1]}}

\title{Towards characterizing the value of edge \\ embeddings in Graph Neural Networks}

\author{Dhruv Rohatgi\\ \small{MIT}\\\texttt{drohatgi@mit.edu} \and Tanya Marwah\\\small{CMU}\\\texttt{tmarwah@andrew.cmu.edu} \and Zachary Chase Lipton\\\small{CMU}\\\texttt{zlipton@cmu.edu} \and Jianfeng Lu\\\small{Duke University}\\\texttt{jianfeng@math.duke.edu} \and Ankur Moitra\\\small{MIT}\\\texttt{moitra@mit.edu} \and Andrej Risteski\\\small{CMU}\\\texttt{aristesk@andrew.cmu.edu}}

\begin{document}

\maketitle

\begin{abstract}
Graph neural networks (GNNs) are the dominant approach to solving machine learning problems defined over graphs. Despite much theoretical and empirical work in recent years, our understanding of finer-grained aspects of architectural design for GNNs remains impoverished. In this paper, we consider the benefits of architectures that maintain and update edge embeddings. On the theoretical front, under a suitable computational abstraction for a layer in the model, as well as memory constraints on the embeddings, we show that there are natural tasks on graphical models for which architectures leveraging edge embeddings can be much shallower. Our techniques are inspired by results on time-space tradeoffs in theoretical computer science.  Empirically, we show architectures that maintain edge embeddings almost always improve on their node-based counterparts---frequently significantly so in topologies that have ``hub'' nodes. 
\end{abstract}
\section{Introduction} 

Graph neural networks (GNNs) have emerged as the dominant approach for solving machine learning tasks on graphs. Over the span of the last decade, many different architectures have been proposed, both in order to improve different notions of efficiency, and to improve performance on a variety of benchmarks. Nevertheless, theoretical and empirical understanding of the impact of different architectural design choices remains elusive. 

One previous line of work \citep{xu2018powerful} has focused on characterizing the representational limitations stemming from the \emph{symmetry-preserving} properties of GNNs when the node features are not informative (also called ``anonymous GNNs'') --- in particular, relating GNNs to the Weisfeiler-Lehman graph isomorphism test \citep{leman1968reduction}. Another line of work \citep{oono2019graph} focuses on the potential pitfalls of the \emph{(over)smoothing effect} of deep GNN architectures, with particular choices of weights and non-linearities, in an effort to explain the difficulties of training deep GNN models. Yet another \citep{black2023understanding} focuses on training difficulties akin to vanishing introduced by \emph{``bottlenecks''} in the graph topology. 

In this paper, we focus on the benefits of maintaining and updating \emph{edge embeddings} over the course of the computation of the GNN. More concretely, a typical way to parametrize a layer $l$ of a GNN \citep{xu2018powerful} is to maintain, for each node $v$ in the graph, a node embedding $h_v^{(l)}$, which is calculated as 
\begin{equation}
    a_v^{(l+1)} = \mbox{AGGREGATE}\Bigl(h_u^{(l)}:u \in N_G(v)\Bigr) \hspace{1cm} h_v^{(l+1)} = \mbox{COMBINE}\Bigl(a_v^{(l+1)}, h_v^{(l)}\Bigr)
    \label{eq:nodempnn}
\end{equation}
where $N_G(v)$ denotes the neighborhood of vertex $v$. These updates can be viewed as implementing a (trained) message-passing algorithm, in which nodes pass messages to their neighbors, which are then aggregated and combined with the current state (i.e., embedding) of a node. The initial node embeddings $h_v^{(0)}$ are frequently part of the task specification (e.g., a vector of fixed features that can be associated with each node). When this is not the case, they can be set to fixed values (e.g., the all-ones vector) or random values.  

But a more expressive way to parametrize a layer of computation is to maintain, for each \emph{edge} $e$, an edge embedding $h_e^{(l)}$ which is calculated as:
\begin{equation} a_e^{(l+1)} = \mbox{AGGREGATE}\Bigl(h_a^{(l)}:a \in M_G(e)\Bigr) \hspace{1cm} h_e^{(l+1)} = \mbox{COMBINE}\Bigl(a_e^{(l+1)}, h_e^{(l)}\Bigr) \label{eq:edgempnn} \end{equation}
where $M_G(e)$ denotes the ``neighborhood'' of edge $e$: that is, all edges $a$ that share a vertex with $e$\footnote{The graph is assumed to be undirected, as is most common in the GNN literature.}. 

This paradigm is at least as expressive as \eqref{eq:nodempnn}: we can simulate a layer of \eqref{eq:nodempnn} by designating the embedding of an edge to be the concatenation of the node embeddings of its endpoints, and noticing that $M_G(e)$ includes all the neighbors of both endpoints of $e$. 
In particular, if a task has natural initial node embeddings, then their concatenations along edges can be used as initial edge embeddings. 
Additionally, there may be tasks where initial features are most naturally associated with edges 
(e.g., attributes of the relationship between two nodes) --- or the final predictions of the network are most naturally associated with edges (e.g., in link prediction, where we want to decide which potential links are true links).

GNNs that fall in the general paradigm of \eqref{eq:edgempnn} have been used for various applications \--- including link prediction \citep{cai2021line, liang2023line} as well as reasoning about relations between objects \citep{battaglia2016interaction}, molecular property prediction \citep{gilmer2017neural, choudhary2021atomistic}, and detecting clusters of communities in graphs \citep{chen2017supervised} \--- with robust empirical benefits. These approaches instantiate the edge-based paradigm in a plethora of ways. However, it is difficult to disentangle to what degree performance improvements come from added information from domain-specific initial edge embeddings, versus properties of the particular architectural choices for the aggregation functions in \eqref{eq:edgempnn}, versus inherent benefits of the edge-based paradigm itself (whether representational, or via improved training dynamics).

We focus on \emph{theoretically and empirically} quantifying the added \emph{representational} benefit from maintaining edge embeddings. Viewing the GNN as a computational model, we can think of the intermediate embeddings as a ``scratch pad''. Since we maintain more information per layer compared to the node-based paradigm \eqref{eq:nodempnn}, we might intuitively hope to be able to use a shallower edge embedding model. However, formally proving depth lower bounds both for general neural networks \citep{telgarsky2016benefits} and for specific architectures \citep{sanford2024representational, sanford2024transformers} frequently requires non-trivial theoretical insights \--- as is the case for our question of interest. In this paper, we show that:
\begin{itemize}[leftmargin=*]
\item 
\emph{Theoretically}, for certain graph topologies, edge embeddings can have substantial \emph{representational} benefits in terms of the depth of the model, when the amount of memory (i.e., total bit complexity) per node or edge embedding is bounded. Our results illuminate some subtleties of using particular lenses to understand design aspects of GNNs: for instance, we prove that taking memory into account reveals depth separations that the classical lens of invariance \citep{xu2018powerful} alone cannot.
\item \emph{Empirically}, when given the same input information, edge-based models almost always lead to performance improvements compared to their node-based counterpart --- and often by a large margin if the graph topology includes ``hub'' nodes with high degree.    
\end{itemize}


\section{Overview of results}

    \subsection{Representational benefits from maintaining edge embeddings.} 

    Our theoretical results elucidate the representational benefits of maintaining edge embeddings. More precisely, we show that there are natural tasks on graphs that can be solved by a \emph{shallow} model maintaining constant-size edge embeddings, but can only be solved by a model maintaining constant-size node embeddings if it is much \emph{deeper}.
    
    To reason about the impact of depth on the representational power of edge-embedding-based and node-embedding-based architectures, we introduce two \emph{local computation models}. In the node-embedding case, we assume each node of the graph $G$ supports a processor that maintains a state with a \emph{fixed amount of memory}. In one round of computation, each node receives messages from the adjacent nodes, which are aggregated by the node into a new state. In this abstraction, we think of the memory of the processor as the total bits of information each embedding can retain, and we think of one round of the protocol as corresponding to one layer of a GNN. The edge-embedding case is formalized in a similar fashion, except that the processors are placed on the edges of the graph, and two edge processors are ``adjacent'' if the edges share a vertex in common. In both cases, the input is distributed across the edges of the graph, and is only locally accessible. 
    
    With this setup in mind, our first result focuses on \emph{probabilistic inference} on graphs, specifically, the task of maximum a-posteriori (MAP) estimation in a pairwise graphical model on a graph $G = (V,E)$. For this task, given edge attributes describing the pairwise interactions $\phi_{\{a,b\}}$, the goal is to compute 
    $\argmax_{x \in \{0,1\}^V} p_{\phi}(x)$, where $p_{\phi}(x) \propto \exp\bigl(\sum_{\{a, b\} \in E} \phi_{\{a,b\}} (x_a, x_b)\bigr).$ 

    \begin{theorem*}[Informal] 
    Consider the task of using a GNN to calculate MAP (maximum a-posteriori) values in a pairwise graphical model, in which the pairwise interactions are given as input embeddings to a node-embedding or edge-embedding architecture. Then, there exists a graph with $O(n)$ vertices and edges, such that: 
    \begin{itemize}[leftmargin=*,itemsep=-0.1em]
        \item Any node message-passing protocol with $T$ rounds and $O(1)$ bits of memory per node processor requires $T = \Omega\left(\sqrt{n}\right)$.
    \item There is an edge message-passing protocol with $O(1)$ rounds and $O(1)$ bits of memory. 
    \end{itemize}

    \end{theorem*}
    
    The proof techniques are of standalone interest: the lower bound on node message-passing protocols is inspired by tracking the ``flow of information'' in the graph, reminiscent of graph pebbling techniques used to prove time-space tradeoffs in theoretical computer science \citep{grigor1976application, abrahamson1991time}. The formal result is \cref{thm:hub-separation}, and the proof sketch is included in \cref{section:map-separation}. 
     
     \paragraph{The view from symmetry.} Above, we are not imposing any \emph{symmetry constraints}  \--- that is, invariance of the computation at a node or edge to its identity and the identities of its neighbors. Indeed, the edge message-passing protocol constructed above is highly non-symmetric. However, we show there is a (different, but also natural) task where \emph{even symmetric} edge message-passing protocols achieve a better depth/memory tradeoff than node message-passing protocols. We state the informal result below; the formal result is \cref{thm:memory-and-symmetry}.

    \begin{theorem*}[Informal]
    Let $n$ be a positive integer. There is a graph $G$ with $O(n)$ vertices and $O(n)$ edges, and a computational task on $G$, such that:
    \begin{itemize}[leftmargin=*,itemsep=-0.1em]
    \item Any node message-passing protocol with $T$ rounds and $O(1)$ bits of memory per node processor requires $T = \Omega(\sqrt{n})$ to solve this task.
    \item There is a symmetric edge message-passing protocol that solves this task with $O(1)$ rounds and $O(1)$ bits of memory.
    \end{itemize}
    \end{theorem*}
    
    \paragraph{Importance of the memory lens.} The memory constraints are crucial for the results above. Without memory constraints, we can show that the node message-passing architecture can simulate the edge message-passing architecture, while only increasing the depth by $1$ (\cref{prop:node-edge-simulation}). Moreover, the \emph{symmetric} node message-passing architecture can simulate the \emph{symmetric} edge message-passing architecture, again while only increasing the depth by $1$.    
    We state the informal result below; the formal result is \cref{t:symmetryalone}.

    \begin{theorem*}[Informal] For any graph $G$, any symmetric edge message-passing protocol on $G$ with $T$ rounds can be represented by a symmetric node message-passing protocol with $T+1$ rounds.
    \end{theorem*}

    We note that unlike prior work that focuses on understanding the representational power of GNN architectures under symmetry constraints \citep{xu2018powerful} --- which requires that the initial node features are the same for all nodes --- our simulation theorem above holds for arbitrary choices of initial node features. 
    
    We view this as evidence that many fine-grained properties of architectural design for GNNs cannot be adjudicated by solely considering them through the lens of symmetries of the network.


\subsection{Empirical benefits of edge-based architectures.}
 
The theory, while only characterizing representational power, suggests that architectures that maintain edge embeddings should have strictly better performance compared to their node embedding counterparts. We verify this in both real-life benchmarks and natural synthetic sandboxes.

First, we consider several popular GNN benchmarks (inspired by both predicting molecular properties, and image-like data), and show that equalizing for all other aspects of the architecture (e.g., depth, dimensionality of the embeddings) --- the accuracy the edge-based architectures achieve is at least as good as their node-based counterparts. Note, the goal of these experiments is \emph{not} to propose a new architecture --- there are already a variety of (very computationally efficient) GNNs that in some manner maintain edge embeddings. The goal is to confirm that --- all other things being equal --- the representational advantages of edge-based architectures do not introduce additional training difficulties. Details are included in \cref{s:benchmarks}.

Next, we consider two synthetic settings to stress test the performance of edge-based architectures. Inspired by the graph topology that provides a theoretical separation between edge and node-based protocols (\cref{thm:hub-separation} and \cref{thm:memory-and-symmetry}), we consider graphs in which there is a hub node, and tasks that are ``naturally'' solved by an edge-based architecture. Precisely, we consider a star graph, in which the labels on the leaves are generated by a ``planted'' edge-based architecture with randomly chosen weights. The node-based architecture, on the other hand, has to pass messages between the leaves indirectly through the center of the star. Empirically, we indeed observe that the performance of edge-based architectures is significantly better. Details are included in \cref{s:stargraph}.

Finally, again inspired by the theoretical setting in \cref{thm:hub-separation}, we consider probabilistic inference on \emph{tree graphs} --- precisely, learning a GNN that calculates node marginals for an Ising model, a pairwise graphical model in which the pairwise interactions are just the product of the end points. An added motivation for this setting is the fact that belief propagation --- a natural algorithm to calculate the marginals --- can be written as an edge-based message-passing algorithm. Again, empirically we see that edge-based architectures perform at least as well as node-based architectures. This advantage is maintained even if we consider ``directed'' versions of both architectures, in which case embeddings are maintained to be sent along each direction of the edge, and the message for the outgoing direction of an edge depends only on the embeddings corresponding to the incoming directions of the edges. Details are included in \cref{s:isingmodels}. 

\section{Related Works}
\paragraph{The symmetry lens on GNNs:} The most extensive theoretical work on GNNs has concerned itself with the representational power of different GNN architectures, while trying to preserve equivariance (to permuting the neighbors) of each layer. \citep{xu2018powerful} connected the expressive power of such architectures to the Weisfeiler-Lehman (WL) test for graph isomorphism. Subsequent works \citep{maron2019provably, zhao2021stars} focused on strengthening the representational power of the standard GNN architectures from the perspective of symmetries---more precisely, to simulate the $k$-WL test, which for $k$ as large as the size of the graph becomes as powerful as testing graph isomorphism. Our work suggests that this perspective may be insufficient to fully understand the representational power of different architectures.

\paragraph{GNNs as a computational machine:} Two recent papers \citep{loukas2019graph, loukas2020hard} considered properties of GNNs when viewed as ``local computation'' machines, in which a layer of computation allows a node to aggregate the current values of the neighbors (in an arbitrary fashion, without necessarily considering symmetries). Using reductions from the CONGEST model, they provide lower bounds on width and depth for the standard node-embedding based architecture. However, they do not consider architectures with edge embeddings, which is a focus of our work.  

\paragraph{Communication complexity methods to prove representational separations:} Tools from distributed computation and communication complexity have recently been applied not only to understand the representational power of GNNs
\citep{loukas2019graph, loukas2020hard}, but also the representational power of other architectures like transformers \citep{ sanford2024representational, sanford2024transformers}. In particular, \citep{sanford2024transformers} draws a connection between number of rounds for a MPC (Massively Parallel Computation) protocol, and the depth of attention-based architectures.  

\paragraph{GNNs for inference and graphical models:} The paper \citep{xu2023rethinking} considers the approximation power of GNNs for calculating marginals for pairwise graphical models, if the family of potentials satisfies strong symmetry constraints. They do not consider the role of edge embeddings or memory. 
    
\section{Setup}

 
\paragraph{Notation.} We will denote the graph associated with the GNN as $G = (V,E)$, denoting the vertex set as $V$ and the edge set as $E$. The graph induces adjacency relations on both edges and nodes, namely for $v,v' \in V$ and $e,e' \in E$, we have: $v\sim v'$ if $\{v,v'\} \in E$; $v \sim e$ if $e = \{u,v\}$ for some $u \in V$; and $e \sim e'$ if $e,e'$ share at least one vertex. For all graphs considered in this paper, we assume that $\{v,v\} \in E$ for all $v \in V$, so that adjacency is reflexive. 
We then define adjacency functions $N_G: V \cup E \to V$ and $M_G: V \cup E \to E$ as $N_G(a) := \{v \in V: a \sim v\}$ and $M_G(a) := \{e \in E: a \sim e\}$. 

\paragraph{Local memory-constrained computation.}
In order to reason about the required depth with different architectures, we will define a mathematical abstraction for one layer of computation in the GNN. We will define two models for local computation, one for each of the edge-embedding and node-embedding architecture. Unlike much prior work on GNNs and distributed computation, we will also have \emph{memory} constraints --- more precisely, we will constrain the bit complexity of the node and edge embeddings being maintained.

In both models, there is an underlying graph $G = (V,E)$, and the goal is to compute a function $g: \Phi^E \to \{0,1\}^V$, where $\Phi$ is the fixed-size \emph{input alphabet},
via several rounds of message-passing on the graph $G$. This domain of $g$ is $\Phi^E$ because in \emph{both} models, the inputs are given on the edges of the graph --- the node model will just be unable to store any \emph{additional} information on the edges. As we will see in \cref{section:map-separation}, this is a natural setup for probabilistic inference on graphs.

In both models, a protocol is parametrized by the number of rounds $T$ required, and the amount of memory $B$ required per local processor. For notational convenience, for $B \in \NN$ we define $\MX_B := \{0,1\}^B$, i.e. the length-$B$ binary strings. Recall that $N_G(v), M_G(v)$ denote the sets of vertices and edges adjacent to vertex $v$ in graph $G$, respectively.

\begin{definition}[Node message-passing protocol]
     Let $T,B \in \NN$ and let $G = (V,E)$ be a graph. A \emph{node message-passing protocol} $P$ on graph $G$ with $T$ rounds and $B$ bits of memory is a collection of functions $(f_{t,v})_{t \in [T], v \in V}$ where $f_{t,v}: \MX_B^{N_G(v)} \times \Phi^{M_G(v)} \to \MX_B$ for all $t,v$. 
    For an \emph{input} $I \in \Phi^E$, the \emph{computation} of $P$ at a round $t \in [T]$ is the map $P_t(\cdot;I): V \to \MX_B$ defined inductively by \[P_t(v;I) := f_{t,v}((P_{t-1}(v';I))_{v' \in N_G(v)}, (I(e))_{e \in M_G(v)})\] where $P_0 \equiv 0$. We say that $P$ \emph{computes} a function $g: \Phi^E \to \{0,1\}^V$ on inputs $\MI \subseteq \Phi^E$ if $P_T(v;I)_1 = g(I)_v$ for all $v\in V$ and all $I \in \MI$.
\label{d:nodeprotocol}
\end{definition}
In words, the value computed by vertex $v$ at round $t$ is some function of the previous values stored at the neighbors $v' \in N_G(v)$, as well as possibly the problem inputs on the edges adjacent to $v$ (i.e. $(I(e))_{e \in M_G(v)})$). Note that $P_t(v;I)$ may indeed depend on $P_{t-1}(v;I)$, due to our convention that $v \in N_G(v)$. We can define the edge message-passing protocol analogously:

\begin{definition}[Edge message-passing protocol] Let $T,B \in \NN$ and let $G = (V,E)$ be a graph. An \emph{edge message-passing protocol} $P$ on graph $G$ with $T$ rounds and $B$ bits of memory is a collection of functions $(f_{t,e})_{t\in[T],e\in E}$ where $f_{t,e}: \MX_B^{M_G(e)} \times \Phi \to \MX_B$ for all $t,e$, together with a collection of functions $(\tilde f_v)_{v \in [V]}$ where $\tilde f_v: \MX_B^{M_G(v)} \to \{0,1\}$. For an \emph{input} $I \in \Phi^E$, the \emph{computation} of $P$ at a timestep $t \in [T]$ is the map $P_t(\cdot;I): E \to \MX_B$ defined inductively by:
\[P_t(e;I) := f_{t,e}((P_{t-1}(e';I))_{e' \in M_G(e)}, I(e))\]
where $P_0 \equiv 0$. We say that $P$ \emph{computes} a function $g: \Phi^E \to \{0,1\}^V$ on inputs $\MI \subseteq \Phi^E$ if $\tilde f_v((P_T(e;I))_{e \in M_G(v)}) = g(I)_v$ for all $v\in V$ and all $I \in \MI$.
\end{definition}

\begin{remark} [Relation to distributed computation literature]
These models are very related to classical models in distributed computation like LOCAL \citep{linial1992locality} and CONGEST \citep{peleg2000distributed}. However, the latter models ignore memory constraints, so we cannot usefully port lower and upper bounds from this literature. 
\end{remark}
\begin{remark}[Computational efficiency]
In the definitions above, we allow the update rules $f_{t,v}, f_{t,e}$ to be arbitrary functions. In particular, a priori they may not be efficiently computable. However, our results showing a function can be implemented by an edge message-passing protocol (\cref{thm:hub-separation}, Part 2 and \cref{thm:memory-and-symmetry}, Part 2) in fact use simple functions (computable in linear time in the size of the neighborhood), implying the protocol can be implemented in parallel (with one processor per node/edge respectively) with parallel time complexity $O(TB \cdot \max_v |M_G(v)|)$. On the other hand, for the results showing a function cannot be implemented by a node message-passing protocol (\cref{thm:hub-separation}, Part 1 and \cref{thm:memory-and-symmetry}, Part 1), we prove an impossibility result for a \emph{stronger} model (one in which the computational complexity of $f_{t,v}$ is unrestricted) --- which makes our results only \emph{stronger}. 
\label{r:computational}
\end{remark}

\paragraph{Symmetry-constrained protocols.} Typically, GNNs are architecturally constrained to respect the symmetries of the underlying graph. Below we formalize the most natural notion of symmetry in our models of computation. Note, our abstraction of a round in the message-passing protocol generalizes the notion of a layer in a graph neural network---and the abstraction defined below correspondingly generalizes the standard definition of permutation equivariance \citep{xu2018powerful}.  We use the notation $\multiset{}$ to denote a multiset.

\begin{definition}[Symmetric node message-passing protocol]\label{def:symmetric-nmp}
A node message-passing protocol $P = (f_{t,v})_{t\in[T],v\in V}$ on graph $G = (V,E)$ is \emph{symmetric} if there are functions $(\fsym_t)_{t \in [T]}$ so that for every $t \in [T]$ and $v \in V$, the function $f_{t,v}$ can be written as:
\[f_{t,v}((c(v'))_{v' \in N_G(v)}, (I(e))_{e \in M_G(v)}) = \fsym_t(c(v), \multiset{(c(v'), I(\{v,v'\})): v' \in N_G(v)}).\]
\end{definition}

\begin{definition}[Symmetric edge message-passing protocol]\label{def:symmetric-emp}
An edge message-passing protocol $P = ((f_{t,e})_{t \in [T], e \in E}, (\tilde f_v)_{v \in V})$ on graph $G = (V,E)$ is \emph{symmetric} if there are functions $(\fsym_t)_{t \in [T]}$ and $\tfsym$ so that for every $t \in [T]$ and $e = \{u,v\} \in E$, the function $f_{t,e}$ can be written as:
\[f_{t,e}((c(e'))_{e' \in M_G(e)}, I(e)) = \fsym_t(I(e), c(e),\multiset{\multiset{c(\{u,v'\}): v' \in N_G(u)}, \multiset{c(\{u',v\}): u' \in N_G(v)}}),\]
and for every $v \in V$,  $\tilde f_v$ can be written as
$\tilde f_v((c(e))_{e \in M_G(v)}) = \tfsym(\multiset{c(e): e \in M_G(v)}).$
\end{definition}

\section{Depth separation between edge and node message passing protocols under memory constraints}\label{section:map-separation}

We will consider a common task in probabilistic inference on a \emph{pairwise graphical model}: calculating the MAP (maximum a-posterior) configuration.  

\begin{definition}[Pairwise graphical model]
For any graph $G = (V,E)$, the \emph{pairwise graphical model} on $G$ with potential functions $\phi_{\{a,b\}}: \{0,1\}^2 \to \RR$ is the distribution $p_\phi \in \Delta(\{0,1\}^V)$ defined as 
\[p_\phi(x) \propto \exp\left(-\sum_{\{a,b\} \in E} \phi_{\{a,b\}}(x_a,x_b)\right).\]
\end{definition}

\begin{definition}[MAP evaluation] Let $\Phi \subseteq \{\phi: \{0,1\}^2 \to \RR\}$ be a finite set of potential functions. A \emph{MAP (maximum a-posteriori) evaluator for $G$} (with potential function class $\Phi$) is any function $g: \Phi^E \to \{0,1\}^V$ that satisfies \[g(\phi) \in \argmax_{x \in \{0,1\}^V} p_\phi(x)\] for all $\phi \in \Phi^E$.
\label{d:mapeval}
\end{definition}

With this setup in mind, we will show that there exists a pairwise graphical model, and a local function class $\Phi$, such that an edge message passing protocol can implement MAP evaluation with a constant number of rounds and a constant amount of memory, while any node message protocol with $T$ rounds and $B$ bits of memory requires 
$TB = \Omega(\sqrt{|V|})$. Precisely, we show:  

\begin{theorem}[Main, separation between node and edge message-passing protocols]\label{thm:hub-separation}
Fix $n \in \NN$. There is a graph $G$ with $O(n)$ vertices and $O(n)$ edges, and a function class $\Phi$ of size $O(1)$, so that:
\begin{enumerate}[leftmargin=*]
    \item Let $g$ be any MAP evaluator for $G$ with potential function class $\Phi$. Any node message-passing protocol on $G$ with $T$ rounds and $B$ bits of memory that computes $g$ requires $TB \geq \sqrt{n}-1$.
    \item There is an edge 
    message-passing protocol $(f_{t,e})_{t,e}$ on $G$ with $O(1)$ rounds and $O(1)$ bits of memory that computes a MAP evaluator for $G$ with potential function class $\Phi$. Additionally, for all $t, e$, the update rule $f_{t,e}$ can be evaluated in $O(|M_G(e)|)$ time. 
\end{enumerate}
\end{theorem}

We provide a proof sketch of the main techniques here, and relegate the full proofs to Appendix~\ref{app:map-separation}. The graph $G$ that exhibits the claimed separation is a disjoint union of $\sqrt{n}$ path graphs, with an additional ``hub vertex'' that is connected to all other vertices in the graph (\cref{fig:path-graph}). The intuition for the separation is that MAP estimation requires information to disseminate from one end of each path to the other, and the hub node is a bottleneck for node message-passing but not edge message-passing. We expand upon both aspects of this intuition below.

\paragraph{Lower bound for node message-passing protocols:} Our main technical lemma for the first half of the theorem is \cref{lemma:light-cone-lb}. It gives a generic framework for lower bounding the complexity of any node message-passing protocol that computes some function $g$, by exhibiting a set of nodes $S \subset V$ where computing $g$ requires large ``information flow'' from distant nodes. More precisely, for any fixed set of ``bottleneck nodes'' $K$, consider the radius-$T$ neighborhood of $S$ when $K$ is removed from the graph. In any $T$-round protocol, input data from outside this neighborhood can only reach $S$ by passing through $K$. But the total number of bits of information computed by $K$ throughout the protocol is only $TB|K|$. This gives a bound on the number of values achievable by $g$ on $S$. We formalize this argument below:  

\begin{lemma}\label{lemma:light-cone-lb}
Let $G = (V,E)$ be a graph. Let $P$ be a node message-passing protocol on $G$ with $T$ rounds and $B$ bits of memory, which computes a function $g: \Phi^E \to \{0,1\}^V$. Pick any disjoint sets $K,S \subseteq V$. Define $H := G[\bar K], F := M_G(N_H^{T-1}(S))$. Then: 
\[TB \geq \frac{1}{|K|} \log \max_{I_F \in \Phi^F} \left|\left\{g_S\left(I_F, I_{\overline{F}}\right): I_{\overline{F}} \in \Phi^{\overline{F}}\right\}\right|.\]
\end{lemma}

\begin{proof}
First, we argue by induction that for each $t \in [T]$ and $v \in V \setminus K$, $P_t(v;I)$ is determined by $I_{M_G(N_H^{t-1}(v))}$ and $(P_\ell(k;I))_{\ell \in [t], k \in K}$. Indeed, by definition, $P_1(v;I)$ is determined by $I_{M_G^1(v)}$ for any $v \in V\setminus K$. For any $t > 1$ and $v \in V\setminus K$, $P_t(v;I)$ is determined by $(P_{t-1}(v';I))_{v' \in N_G(v)}$ and $(I(e))_{e \in M_G(v)}$. Note that $N_G(v) \subseteq N_H(v) \cup K$. Thus, using the induction hypothesis for each $v' \in N_H(v)$, we get that $(P_{t-1}(v';I))_{v' \in N_G(v)}$ is determined by $\bigcup_{v' \in N_H(v)} I_{M_G(N_H^{t-2}(v'))}$ and $(P_\ell(k;I))_{\ell \in [t], k \in K}$. So $P_t(v;I)$ is determined by $I_{M_G(N_H^{t-1}(v))}$ and $(P_\ell(k;I))_{\ell \in [t], k \in K}$, completing the induction.

Since $P$ computes $g$ and $S \subseteq V \setminus K$, we get that $g_S(I)$ is determined by $I_{M_G(N_H^{T-1}(S))} = I_F$ and $(P_\ell(k;I))_{\ell\in[T],k\in K}$. Thus, for any fixed $I_F \in \Phi_F$, we have
\[\left|\left\{g_S\left(I_F, I_{\overline{F}}\right): I_{\overline{F}} \in \Phi^{\overline{F}}\right\}\right| \leq \left|\left\{(P_\ell(k;(I_F,I_{\overline{F}})))_{\ell \in [T], k \in K}): I_{\overline{F}} \in \Phi^{\overline{F}}\right\}\right| \leq |\MX_B|^{T|K|} = 2^{TB|K|}.\]
The lemma follows.
\end{proof}

\begin{remark}
    The proof technique is inspired by and related to classic techniques (specifically, Grigoriev's method) for proving time-space tradeoffs for restricted models of computation like branching programs (\citep{grigor1976application}, see Chapter 10 in \cite{savage1998models} for a survey). There, one defines the ``flow'' of a function, which quantifies the existence of subsets of coordinates, such that setting them to some value, and varying the remaining variables results in many possible outputs. In our case, the choice of subsets is inherently tied to the topology of the graph $G$. Our technique is also inspired by and closely related to the ``light cone'' technique for proving round lower bounds in the LOCAL computation model \citep{linial1992locality}. However, our technique takes advantage of bottlenecks in the graph to prove stronger lower bounds (which would be impossible in the LOCAL model where memory constraints are ignored).
\end{remark}

The proof of Part 1 of \cref{thm:hub-separation} now follows from an application of \cref{lemma:light-cone-lb} with a particular choice of $K$ and $S$. Specifically, we choose $K$ to be the ``hub'' node (i.e. $K = \{0\}$) and $S$ to be the set of left endpoints of each path. To show that any MAP evaluator has large information flow to $S$ (in the quantitative sense of \cref{lemma:light-cone-lb}), it suffices to observe that in a pairwise graphical model on $G$ where a different external field is applied to the right endpoint of each path, and all pairwise interactions along paths are positive, the MAP estimate on each vertex in $S$ must match the external field on the corresponding right endpoint.

\paragraph{Upper bound for edge message-passing protocols:} The key observation for constructing a constant-round edge message-passing protocol for MAP estimation on $G$ is that all of the input data can be collected on the edges adjacent to the hub vertex. At this point, every such edge has access to all of the input data, and hence can evaluate the function. If $G$ were an arbitrary graph, this final step would potentially be NP-hard. However, since the induced subgraph after removing the hub vertex is a disjoint union of paths, in fact there is a linear-time dynamic programming algorithm for MAP estimation on $G$ (\cref{lemma:evaluate-g}). This completes the proof overview for \cref{thm:hub-separation}; we now provide the formal proof.

\begin{proof}[Proof of \cref{thm:hub-separation}]
Let $G$ be the graph on vertex set $V := \{0\} \cup [\sqrt{n}]\times [\sqrt{n}]$ with edge set defined below (see also \cref{fig:path-graph}):
\[E := \{\{0, (i,j)\}: i,j \in [\sqrt{n}]\} \cup \{\{(i,j),(i+1,j)\}: 2 \leq i \leq \sqrt{n}, 1 \leq j \leq \sqrt{n}\}.\]
Define
\[\Phi := \{(x_a,x_b) \mapsto \mathbbm{1}[x_a \neq x_b], (x_a,x_b) \mapsto \mathbbm{1}[x_a \neq 1 \lor x_b \neq 1], (x_a,x_b) \mapsto \mathbbm{1}[x_a \neq 0 \lor x_b \neq 0], (x_a,x_b) \mapsto 0\}.\]
First, let $g:\Phi^E \to \{0,1\}^V$ be any MAP evaluator for $G$ with potential function class $\Phi$, and consider any node message-passing protocol on $G$ with $T$ rounds and $B$ bits of memory that computes $g$. Let $K = \{0\}$ and $S = \{(1,j): j \in [\sqrt{n}]\}$. Suppose that $T \leq \sqrt{n}-2$. Let $F := M_G(N_H^{T-1}(S))$ and note that $\{(\sqrt{n}-1,j), (\sqrt{n},j)\} \not \in F$ for all $j \in [\sqrt{n}]$. Let $I_F: F \to \Phi$ be the mapping that assigns the function $(x_a,x_b) \mapsto 0$ to each edge $\{0,(i,j)\} \in F$ and $(x_a,x_b) \mapsto \mathbbm{1}[x_a \neq x_b]$ to each edge $\{(i,j),(i+1,j)\} \in F$. We claim that 
\[\left|\left\{g_S(I_F,I_{\overline F}): I_{\overline F} \in \Phi^{\overline{F}}\right\}\right| \geq 2^{\sqrt{n}}.\]
Indeed, for any string $y \in \{0,1\}^{\sqrt{n}}$, consider the mapping $I_{\overline F}: \overline{F} \to \Phi$ that assigns the function $(x_a,x_b) \mapsto \mathbbm{1}[x_a \neq y_j \lor x_b \neq y_j]$ to each edge $\{(\sqrt{n}-1,j),(\sqrt{n},j)\} \in F$, assigns $(x_a,x_b) \mapsto 0$ to each edge $\{0,(i,j)\} \in E\setminus F$, and assigns $(x_a,x_b) \mapsto \mathbbm{1}[x_a \neq x_b]$ to all remaining edges in $E\setminus F$. Then every minimizer of
\[\min_{x \in \{0,1\}^V} \sum_{\{a,b\} \in E} I_{\{a,b\}}(x_a,x_b)\]
satisfies $x_{(1,j)} = \dots = x_{(\sqrt{n},j)} = y_j$ for all $j \in [\sqrt{n}]$. Hence, $g_S(I_F,I_{\overline F}) = y$. Since $y$ was chosen arbitrarily, this proves the claim. But now \cref{lemma:light-cone-lb} implies that $TB \geq \sqrt{n}$.

We now construct an edge message-passing protocol $P$ on $G$ with $T = 3$ and $B = 4$. We (arbitrarily) identify $\Phi$ with $\{0,1\}^2$. For all $i,j \in \sqrt{n}$, define
\begin{align*}
f_{1,\{(i,j),(i+1,j)\}}(x, y) &:= y & \text{ if } i < \sqrt{n}\\
f_{2,\{0,(i,j)\}}(x, y) &:= (x_{\{(i,j),(i+1,j)\}}, x_{\{0,(i,j)\}}) & \text{ if } i < \sqrt{n} \\
f_{3,\{0,(i,j)\}}(x, y) &:= (g_0(J(x)), g_{(i,j)}(J(x)))
\end{align*}
where the second line is well-defined since edge $\{0,(i,j)\}$ is adjacent to both itself and edge $\{(i,j),(i+1,j)\}$; and in the third line the function is computing $g_0$ and $g_{(i,j)}$ on the input $J(x) \in \Phi^E$ defined as 
\[J(x)_e := \begin{cases} 
(x_{\{0,(k,\ell)\}})_{1:2} & \text{ if } e = \{(k,\ell),(k+1,\ell)\} \\
(x_{\{0,(k,\ell)\}})_{3:4} & \text{ if } e = \{0, (k,\ell)\}
\end{cases},
\]
where we use the notation $v_{a:b}$ for a vector $v$ and indices $a,b \in \NN$ to denote $(v_a,v_{a+1},\dots,v_b)$. Note that $J(x)$ is a well-defined function of $x$ for every edge $\{0,(i,j)\}$, because $\{0,(i,j)\} \sim \{0,(k,\ell)\}$ for all $i,j,k,\ell \in [n]$. Finally, define all other functions $f_{t,e}$ to compute the all-zero function, and define
\[\tilde f_v(x) := \begin{cases} 
(x_{\{0,(1,1)\}})_{1:2} & \text{ if } v = 0 \\ 
(x_{\{0, v\}})_{3:4} & \text{ otherwise}
\end{cases}.\]
This function is well-defined since $v=0$ is adjacent to edge $\{0,(1,1)\}$ and any vertex $v \in V\setminus\{0\}$ is adjacent to edge $\{0,v\}$.

Fix any $I \in \Phi^E$. From the definition, it's clear that $P_2(\{0,(i,j)\}; I) = (I_{\{(i,j),(i+1,j)\}}, I_{\{0,(i,j)\}})$ for all $I$ and $(i,j) \in [\sqrt{n}-1]\times[\sqrt{n}]$. Hence $J((P_2(e';I))_{e' \in M_G(e)})_e = I$ for all edges $e$ of the form $(0,\{i,j\})$, and so $P_3(\{0,(i,j)\};I) = (g_0(I), g_{(i,j)}(I))$ for all $(i,j) \in [\sqrt{n}]\times[\sqrt{n}]$. This means that $\tilde f_v((P_3(e;I))_{e \in M_G(v)}) = g(I)_v$ for all $v \in V$, so the protocol indeed computes $g$.

It remains to argue about the computational complexity of the updates $f_{t,e}$. It's clear that for all $e \in E$ and $t \in \{1,2\}$, the function $f_{t,e}$ can be evaluated in input-linear time. The only case that requires proof is when $t = 3$ and $e = \{0,(i,j)\}$ for some $i,j \in \sqrt{n}$. In this case $|M_G(e)| = \Theta(n)$, so it suffices to give an algorithm for evaluating the function $g: \Phi^E \to \{0,1\}^V$ on an explicit input $J$ in $O(n)$ time. This can be accomplished via dynamic programming (\cref{lemma:evaluate-g}).
\end{proof}

\begin{remark}
A quantitatively stronger (and in fact tight) separation is possible if one considers general tasks rather than MAP estimation tasks \--- see \cref{app:cc}.
\end{remark}

The separation discussed above crucially relies on the existence of a high-degree vertex in $G$. When the maximum degree of $G$ is bounded by some parameter $\Delta$, it turns out that any edge message-passing protocol can be simulated by a node message-passing protocol with roughly the same number of rounds and only a $\Delta$ factor more memory per processor. The idea is for each node to simulate the computation that would have been performed (in the edge message-passing protocol) on the adjacent edges. The following proposition formalizes this idea (proof in \cref{app:map-separation}):

\begin{proposition}\label{prop:node-edge-simulation}
Let $T, B \geq 1$. Let $G = (V,E)$ be a graph with maximum degree $\Delta$. Let $P$ be an edge message-passing protocol on $G$ with $T$ rounds and $B$ bits of memory. Then there is a node message-passing protocol $P'$ on $G$ that computes $P$ with $T+1$ rounds and $O(\Delta B)$ bits of memory.
\end{proposition}

\section{Depth separation under memory and symmetry constraints}\label{section:new}

One drawback of the separation in the previous section is that the constructed edge protocol was highly non-symmetric, whereas in practice GNN protocols are typically architecturally constrained to respect the symmetries of the underlying graph. In this section we prove that there is a separation between the memory/round trade-offs for node and edge message-passing protocols even under additional symmetry constraints.

\begin{theorem}\label{thm:memory-and-symmetry}
Let $n \in \NN$. There is a graph $G = (V,E)$ with $O(n)$ vertices and $O(n)$ edges, and a function $g: \{0,1\}^E \to \{0,1\}^V$, so that:
\begin{enumerate}[leftmargin=*]
    \item Any node message-passing protocol on $G$ with $T$ rounds and $B$ bits of memory that computes $g$ requires $TB \geq \Omega(\sqrt{n})$.
    \item There is a symmetric edge message-passing protocol on $G$ with $O(1)$ rounds and $O(\log n)$ bits of memory that computes $g$. 
\end{enumerate}
\end{theorem}

For intuition, we start by sketching the proof of a relaxed version of the theorem, where the input alphabet is $[n]$ instead of $\{0,1\}$. We then discuss how to adapt the construction to binary alphabet. 

\paragraph{Large-alphabet construction.} Let $G = (V, E)$ be a star graph with root node $0$ and leaves $\{1,\dots,n\}$. We define a function $g: [n]^E \to \{0,1\}^V$ by $g(I)_v = 1$ if and only if there is some edge $e \neq \{0, v\}$ such that $I(e) = I(\{0,v\})$, i.e. the input on edge $\{0,v\}$ equals the input on some other edge. Since $g$ is defined to be equivariant to relabelling the edges, and all edges are incident to each other, it is straightforward to see that there is a symmetric one-round edge message-passing protocol that computes $g$ with $O(\log n)$ memory (in contrast, the edge message-passing protocol constructed in \cref{section:map-separation} was not symmetric, as it required that the edges incident to the high-degree vertex were labelled by which path they belonged to). However, there is no low-memory, low-round \emph{node} message-passing algorithm. Informally, this is because vertex $0$ is an information bottleneck, and $\Omega(n)$ bits of information need to pass through it. Similar to in \cref{section:map-separation}, this intuition can be made formal using \cref{lemma:light-cone-lb}.

\paragraph{Modifying for small alphabet.} The large alphabet size seems crucial to the above construction: if we were to naively modify the above construction so that each edge takes input in $\{0,1\}$ (without changing the graph topology or the function $g$), then there \emph{would} be a low-memory, low-round message-passing protocol, since the root node simply needs to compute the histogram of the leaves' inputs, which takes space $O(\log n)$. Each leaf node can use this information together with its own input value to compute its output. Essentially, there is no information bottleneck because there is a concise, sufficient ``summary" of the input data.

However, the above construction can in fact be adapted to work with binary alphabet, by modifying the graph topology. At a high level, for each leaf node $u$ in the above construction, we add $n$ descendants and encode the input that was originally on $u$ on the descendants of $u$, in unary. Of course, this new graph has $n^2$ nodes, so we must rescale parameters accordingly.

We now make this idea formal. For notational convenience, define $m = \lfloor \sqrt{n}\rfloor$. We define a graph $G = (V,E)$ that is a perfect $n$-ary tree of depth two. Formally, the graph $G$ has vertex set $V = \{0\} \cup [m] \cup ([m] \times [m])$. Vertex $0$ is adjacent to each $i \in [m]$, and each $i \in [m]$ is additionally adjacent to $(i,j)$ for all $j \in [m]$. We define a function $g: \{0,1\}^E \to \{0,1\}^V$ as follows. On input $I \in \{0,1\}^E$, for each edge $e \in E$, define the \emph{input summation} at $e$ to be \[C(I)_e := \sum_{e' \in M_G(e)} I(e').\] Intuitively, one may think of $C(I)_e$ as simulating the input on $e$ in the ``large alphabet'' construction described in \cref{section:new}. Next, define
\begin{align*}
g(I)_{(u,j)} &:= 0. \\ 
g(I)_u &:= \mathbbm{1}[\#|e \in M_G(\{0,u\}): C(I)_e = C(I)_{\{0,u\}}| > m+1]. \\ 
g(I)_0 &:= \mathbbm{1}[\exists u \in [m]: g(I)_u = 1].
\end{align*}

In words, $g(I)_u$ is the indicator for the event that, among the $2m+1$ edges adjacent to $\{0,u\}$ (which include $\{0,u\}$ itself), more than $m+1$ edges have the same input summation as $\{0,u\}$. At a high level, this definition of $g$ was designed to satisfy three criteria. First, $g(I)_u$ depends on the input values on other branches of the tree: in particular, if $I_{\{0,v\}} = 0$ for all $v \in [n]$, then $C(I)_e = C(I)_{\{0,u\}}$ for all edges $e$ in the subtree of $u$, so $g(I)_u$ exactly measures the event that there is \emph{at least one} edge $e$ outside the subtree of $u$ for which $C(I)_e = C(I)_{\{0,u\}}$. Second, there is no concise ``summary'' of $I$ such that $g(I)_u$ can be determined from this summary in conjunction with the inputs on the subtree of $u$. Third, $g(I)$ is equivariant to re-labelings of the tree. 

The first two criteria, together with the fact that the root vertex $0$ is an ``information bottleneck'' for $G$, can be used to show that any node message-passing algorithm that computes $g$ on $G$ requires either large memory or many rounds. The third criterion enables construction of a symmetric edge message-passing protocol for $g$. The arguments are formalized in the claims below.

\begin{claim}\label{claim:ms-lb}
For graph $G$ and function $g$ as defined above, any node message-passing protocol on $G$ that computes $g$ with $T$ rounds and $B$ bits of memory requires $TB \geq \Omega(m)$.
\end{claim}

\begin{proof}
Consider any input $I \in \{0,1\}^E$ with $I(\{0,u\}) = 0$ for all $u \in [m]$. Then for any $u,j \in [m]$, we have \[C(I)_{\{u,(u,j)\}} = C(I)_{\{0,u\}} = \sum_{i=1}^m I(\{u,(u,i)\}).\] Thus $g(I)_u = 1$ if and only if there exists some $v \in [m]\setminus\{u\}$ with $C(I)_{\{0,u\}} = C(I)_{\{0,v\}}$, or equivalently $\sum_{i=1}^m I(\{u,(u,i)\}) = \sum_{i=1}^m I(\{v,(v,i)\})$. 

Fix $T,B$ and suppose that $P$ is a node message-passing protocol on $G$ that computes $g$ with $T$ rounds and $B$ bits of memory. Define sets of vertices $K := \{0\}$ and $S := \{1,\dots,m/2\}$. Let $H := G[\overline K]$ and $F := M_G(N^{T-1}_H(S))$. Then for any $T$, we have that \[F = \{\{0,u\}: 1 \leq u \leq m/2\} \cup \{\{u,(u,j)\}: 1 \leq u \leq m/2, 1 \leq j \leq m\}.\] Define a vector $I_F \in \Phi^F$ by 
\begin{align*}
I_{\{0,u\}} &= 0 \text{ for } 1 \leq u \leq m/2 \\ 
I_{\{u,(u,j)\}} &= \mathbbm{1}[j \leq u] \text{ for } 1 \leq u \leq m/2, 1 \leq j \leq m.
\end{align*}
Now fix any $x \in \{0,1\}^S$. We claim that there is some $I_{\overline F} \in \Phi^{\overline F}$ such that $g_S(I_F,I_{\overline F}) = x$. Indeed, let us define $I_{\overline F}$ by:
\begin{align*}
I_{\{0,v\}} &= 0 \text{ for } m/2 < v \leq m \\ 
I_{\{v,(v,j)\}} &= x_{v-m/2}\mathbbm{1}[j \leq v-m/2] \text{ for } m/2 < v \leq m, 1 \leq j \leq m.
\end{align*}
Then $C(I)_{\{0,u\}} = u$ for all $1 \leq u \leq m/2$, and $C(I)_{\{0,v\}} = (v-m/2) x_{v-m/2}$ for all $m/2 < v \leq m$. It follows that for any $1 \leq u \leq n/2$, $x_u = 1$ if and only if there exists some $v \in [m]\setminus u$ with $C(I)_{\{0,u\}} = C(I)_{\{0,v\}}$, and hence $x_u = g(I)_u$. We conclude that 
\[\left|\left\{g_S(I_F,I_{\overline F}): I_{\overline F} \in \Phi^{\overline F}\right\}\right| \geq 2^{m/2}.\]
Applying \cref{lemma:light-cone-lb} we conclude that $TB \geq \Omega(m)$ as claimed.
\end{proof}

\begin{claim}\label{claim:ms-ub}
For graph $G$ and function $g$ as defined above, there is a symmetric edge message-passing protocol on $G$ that computes $g$ with $O(1)$ rounds and $O(\log m)$ bits of memory.
\end{claim}

\begin{proof}
In the first round, each edge processor reads its input value. In the second round, each edge processor sums the values computed by all neighboring edges (including itself). In the third round, each edge processor computes the indicator for the event that strictly more than $m+1$ neighboring edges (including itself) have the same value as itself. In the final aggregation round, the output of a vertex is the indicator for the event that any neighbor has value $1$.

By construction, the value computed by any edge $e$ after the second round is exactly $C(I)_e$. Thus, after the third round, the value computed by any edge $\{0,u\}$ is exactly $g(I)_u$. Moreover, the value computed by any edge $\{u,(u,j)\}$ is $0$ after the third round, since such edges only have $m+1$ neighbors. It follows by construction of the final aggregation step that the protocol computes $g$.
\end{proof}

\begin{proof}[Proof of \cref{thm:memory-and-symmetry}]
Immediate from \cref{claim:ms-lb,claim:ms-ub}.
\end{proof}

\section{Symmetry alone provides no separation}\label{section:mapsymmetry}

In the previous sections we saw that examining \emph{memory constraints} yields a separation between different GNN architectures (whether or not we take symmetry into consideration). In this section, we consider what happens if we solely consider \emph{symmetry constraints} (that is, constraints imposed by requiring that the computation in a round of the protocol is invariant to permutations of the order of the neighbors). This viewpoint was initiated by \cite{xu2018powerful}, who showed that when the initial node features are uninformative (that is, the same for each node), a standard GNN necessarily outputs the same answer for two graphs that are 1-Weisfeiler-Lehman equivalent (that is, graphs that cannot be distinguished by the Weisfeiler-Lehman test, even though they may not be isomorphic). 

To be precise, we revisit the representational power of symmetric GNN architectures in the setting where the input features may be distinct and informative.
We show that \emph{if we remove the memory constraints} from \cref{section:map-separation}, but \emph{impose permutation invariance} for the computation in each round, any function that is computable by a $T$-layer edge message-passing protocol can be computed by a $(T+1)$-layer node message-passing protocol. Note that this statement is incomparable to \cref{prop:node-edge-simulation} because we impose constraints on symmetry, but remove constraints on memory. 



\begin{theorem}[No separation under symmetry constraints]
Let $T \geq 1$. Let $P$ be a symmetric edge message-passing protocol (Definition~\ref{def:symmetric-emp}) on graph $G = (V,E)$ with $T$ rounds. Then there is a $(T+1)$-round symmetric node message-passing protocol (Definition~\ref{def:symmetric-nmp}) $P'$ on $G$ that computes the same function as $P$.
\label{t:symmetryalone}
\end{theorem}
\begin{remark}
    \cref{t:symmetryalone} and its proof are closely related to the fact that the $1$-Weisfeiler-Lehman test is equivalent to the $2$-Weisfeiler-Lehman test, which was reintroduced in the context of higher-order GNNs \citep{huang2021short}. However, the $k$-Weisfeiler-Lehman test only characterizes the representational power of $k$-GNNs with uninformative input features (i.e. that are identical for all nodes). \cref{t:symmetryalone} shows that even with arbitrary input features on the edges, the computation of a GNN with edge embeddings and symmetric updates can be simulated by a GNN with only node embeddings, without losing symmetry.
\end{remark}

To prove \cref{t:symmetryalone}, note that it suffices to simulate the protocol $P$ for which the update rules $\fsym, \tfsym$ in \cref{def:symmetric-emp} are identity functions on the appropriate domains. In order to simulate $P$, we construct a symmetric node message-passing protocol $P'$ for which the computation at time $t+1$ and node $v$ on input $I$ is the multiset of features computed by $P$ at time $t$ at edges adjacent to $v$: $Q_t(v;I) := \multiset{P_t(e;I): e \in M_G(v)}.$
This is possible since the computation of $P$ at time $t$ and edge $e = (u,v)$ is
$P_t(e;I) 
= (I(e),P_{t-1}(e;I),\multiset{Q_{t-1}(u;I),Q_{t-1}(v;I)}).$ 
The node message-passing protocol is tracking $Q_{t-1}(\cdot;I)$; moreover, it can recursively compute $P_{t-1}(e;I)$ using the same formula. See \cref{app:mapsymmetry} for the formal proof.

\section{Empirical benefits of edge-based architectures}
\label{sec:experiments}
In this section we demonstrate that the representational advantages the theory suggests are borne out by experimental evaluations, both on real-life benchmarks and two natural synthetic tasks we provide.  
Note that all the experiments were done 
on a machine with $8$ Nvidia A6000 GPUs.

\subsection{Performance on common benchmarks}
\label{s:benchmarks}
First we compare the performance of the most basic GNN architecture (Graph Convolutional Network, \cite{kipf2016semi}) with node versus edge embeddings. In the notation of \eqref{eq:nodempnn} and \eqref{eq:edgempnn}, the AGGREGATE and COMBINE operations are integrated as a transformation that looks like \cref{eq:nodegcn} or \cref{eq:edgegcn}:\footnote{This is the ``residual'' parametrization, which we use in experiments unless otherwise stated.} 
\begin{align}
    h^{(l+1)}_{v} &= h^{(l)}_{v} + \sigma\bigl(W^{(l)}  \mbox{MEAN}\bigl(h^{(l)}_{w}: w \in N_G(v) \setminus \{v\}  \bigr)\bigr) \label{eq:nodegcn}\\
    h^{(l+1)}_{e} &= h^{(l)}_{e} + \sigma\bigl(W^{(l)}  \mbox{MEAN}\bigl(h^{(l)}_{f}: f \in M_G(e) \setminus \{e\}  \bigr)\bigr) \label{eq:edgegcn}
\end{align}
for trained matrices $W^{(l)}$ and a choice of non-linearity $\sigma$. 
The only difference between these architectures is that in the latter case, the message passing happens over the \emph{line graph} of the original graph (i.e. the neighborhood of an edge is given by the other edges that share a vertex with it) --- thus, this can be viewed as an ablation experiment in which the only salient difference is the type of embeddings being maintained. To also equalize the information in the input embeddings, we only use the node embeddings in the benchmarks we consider: for the edge-based architecture \eqref{eq:edgempnn}, we initialize the edge embeddings by the concatenation of the node embeddings of the endpoints.  

In \cref{tab:benchmark-table}, we show that \emph{this single change} (without any other architectural modifications) uniformly results in the edge-based architecture at least matching the performance of the node-based architecture, sometimes improving upon it. \emph{Note, the purpose of this table is not to advocate a new GNN architecture}\footnote{In particular, the edge-based architecture is often much more computationally costly to evaluate.}---  
but to confirm that the increased representational power of the edge-based architecture indicated by the theory also translates to improved performance when the model is trained.  
For each benchmark, we follow the best performing training configuration as delineated in \citep{dwivedi2023benchmarking}.

\begin{table*}[th!]
   \centering
   \resizebox{\linewidth}{!}{
   \begin{tabular}{cccccc}
    \toprule
    \multirow{3}{*}{Model} & ZINC & MNIST & CIFAR-10 & Peptides-Func & Peptides-Struct 
    \\
    \cmidrule(lr){2-6}
    & MAE ($\downarrow)$ & ACCURACY ($\uparrow$) & ACCURACY ($\uparrow$) & AP ($\uparrow)$ & MAE ($\downarrow$) \\
    \midrule 
    GCN & $0.3430\pm 0.034$ & \bm{$95.29 \pm 0.163$} & $55.71 \pm 0.381$ & $0.6816\pm 0.007$ & $0.2453 \pm 0.0001$ \\
    
    Edge-GCN (Ours)& \bm{$0.3297\pm 0.011$} & $94.37 \pm 0.065$ & \bm{$57.44 \pm 0.387$} & \bm{$0.6867\pm 0.004$} & \bm{$0.2437 \pm 0.0005$} \\
    \bottomrule
   \end{tabular}
   }
    \caption{Comparison of node-based \eqref{eq:nodegcn} and edge-based \eqref{eq:edgegcn} GCN architectures across various graph benchmarks. The performance of the edge-based architecture robustly matches or improves the node-based architecture.}
    \label{tab:benchmark-table}
\end{table*}
\vspace{-0.3cm}
\subsection{A synthetic task for topologies with node bottlenecks}
\label{s:stargraph}

 The topologies of the graphs in \cref{thm:hub-separation} and \cref{thm:memory-and-symmetry} both involve a ``hub'' node, which is connected to all other nodes in the graph. Intuitively, in node-embedding architectures, such nodes have to mediate messages between many pairs of other nodes, which is difficult when the node is constrained by memory. To empirically stress test this intuition, we produce a synthetic dataset and train a GNN to solve a regression task on a graph with a fixed \emph{star-graph} topology---a simpler topology than the constructions in \cref{thm:hub-separation} and \cref{thm:memory-and-symmetry}---but capturing the core aspect of both. A star graph is a graph with a center node $v_0$, a set of $n$ leaf nodes $\{v_i\}_{i \in [n]}$, and edge set $\{\{v_0, v_i\}_{i \in [n]}\}$. A training point in the dataset is a list $(x_i,y_i)_{i=1}^n$ where $x_i$ is the \emph{input feature} and $y_i$ is the \emph{label} for leaf node $v_i$.

The input features are in $\mathbb{R}^{10}$, and sampled from a standard Gaussian. The labels $y_i$ are produced as outputs of a \emph{planted} edge-based architecture. Namely, for a standard edge-based GCN as in \eqref{eq:edgegcn}, we randomly choose values for the matrices $\{W_{i}\}_{i \in [k]}$ for some number of layers $k$, and set the labels to be the output of this edge-based GCN, when the initial edge features to the GCN are set as $h^{(0)}_{\{v_0, v_i\}} := x_i$, i.e. the input feature $x_i$ at the corresponding leaf $i$. 
In \cref{tab:stargraph-table-e}, we show the performance of edge-based and node-based architectures on this dataset, varying the number of leaves $n$ in the star graph and the depth $k$ of the planted edge-based model. In each case, the numbers indicate RMSE of the best-performing edge-based and node-based architecture, sweeping over depths up to 10 ($2 \times$ the planted model), widths $\in \{16, 32, 64\}$, and a range of learning rates.

Since the planted edge-based model satisfies both \emph{invariance} constraints (by design of the GCN architecture) and \emph{memory} constraints (since the planted model maintains 10-dimensional embeddings), we view these results as empirical corroboration of \cref{thm:memory-and-symmetry}---and even for simpler topologies than the proof construction.  

\begin{table}[htbp]
\centering
\sisetup{table-format=1.4}
\begin{tabular}{@{}r*{6}{S[table-format=1.4]}@{}}
\toprule
& \multicolumn{6}{c}{Depth of Planted Model (RMSE)} \\
\cmidrule(l){2-7}
\multirow{2}{*}{\makecell[r]{Number of\\Leaves}} & \multicolumn{2}{c}{5} & \multicolumn{2}{c}{3} & \multicolumn{2}{c}{1} \\
\cmidrule(lr){2-3} \cmidrule(lr){4-5} \cmidrule(l){6-7}
& {Edge} & {Node} & {Edge} & {Node} & {Edge} & {Node} \\
\midrule
64 & 0.004  & 0.3790  & 0.011 & 0.3596 & 0.008 & 0.3752 \\
32 & 0.003 & 0.3664 & 0.005 & 0.3626 & 0.003 & 0.3614 \\
16 & 0.007 & 0.3336 & 0.002 & 0.2100 & 0.002 & 0.2847 \\
\bottomrule
\end{tabular}
\caption{Performance (in RMSE $\downarrow$) of edge-based and node-based architectures on a star-graph topology. The first number is the performance of the best edge-based model, and the second is the best node-based model, across a range of depths up to 10 ($2 \times$ the planted model), widths $\in \{16, 32, 64\}$, and a range of learning rates.}
\label{tab:stargraph-table-e}
\end{table}

\subsection{A synthetic task for inference in Ising models}
\label{s:isingmodels}
Finally, motivated by the probabilistic inference setting in \cref{thm:hub-separation}, we consider a synthetic sandbox of using GNNs to predict the values of marginals in an Ising model~\citep{ising1925beitrag,onsager1944crystal} \--- a natural type of pairwise graphical model where each node takes a value in  $\{\pm 1\}$, and each edge potential is a weighted product of the edge endpoint values. Concretely, the probability distribution 
of an Ising model over graph $G = (V,E)$ has the form: 
$$\forall x \in \{\pm 1\}^n:
    p_{J,h}(x) \propto \exp\Bigl(\sum_{\{i,j\} \in E} J_{\{i,j\}} x_i x_j + \sum_{i \in V} h_i x_i \Bigr).$$

Similar to in \cref{s:stargraph}, we construct a training set where the graph $G$ and and edge potentials stay fixed (precisely, $J_{i,j} = 1$  for all $\{i,j\} \in E$). A training data-point consists of a vector of node potentials $\{h_i\}_{i \in [n]}$, and labels $\{\mathbb{E}[x_i]\}_{i \in [n]}$ consisting of the marginals from the resulting Ising model $p_{J,h}$. The node potentials are sampled from a standard Gaussian distribution.     

There is a natural connection between GNNs and calculating marginals: a classical way to calculate $\{\mathbb{E}[x_i]\}$ when $G$ is a \emph{tree} is to iterate a message passing algorithm called \emph{belief propagation} \eqref{eq:bpupdates}, in which for each edge $\{i,j\}$ and direction $i \to j$, a message $\nu^{(t+1)}_{i \to j}$ is calculated that depends on messages $\{\nu^{(t)}_{k \to i}\}_{\{k,i\} \in E}$. The belief-propagation updates \eqref{eq:bpupdates} naturally fit the general edge-message passing paradigm from \eqref{eq:edgempnn}. In fact, they fit even more closely a ``directed'' version of the paradigm, in which each edge $\{i,j\}$ maintains two embeddings $h_{i \to j}, h_{j \to i}$, such that the embedding for direction $h_{i \to j}$ depends on the embeddings $\{h_{k \to i}\}_{\{k,i\} \in E}$ --- and it is possible to derive a similar ``directed'' node-based architecture (See \cref{s:archsising}). For both the undirected and directed version of the architecture, we see that maintaining edge embeddings gives robust benefits over maintaining node embeddings---for a variety of tree topologies including complete binary trees, path graphs, and uniformly randomly sampled trees of a fixed size. More details are included in \cref{a:isingdetails}.
\vspace{-0.2cm}
\section{Conclusions and future work}

Graph neural networks are the best-performing machine learning method for many tasks over graphs. There is a wide variety of GNN architectures, which  frequently make opaque design choices and whose causal influence on the final performance is difficult to understand and estimate. In this paper, we focused on understanding the impact of maintaining edge embeddings on the representational power, as well as the subtleties of considering constraints like memory and invariance. One significant downside of maintaining edge embeddings is the \emph{computational} overhead on dense graphs. Hence, a fruitful direction for future research would be to explore more computationally efficient variants of edge-based architectures that preserve their representational power and performance.   

\section*{Acknowledgements}

DR is supported by a U.S. DoD NDSEG Fellowship. TM is supported in part by CMU Software Engineering Institute via Department of Defense under contract
FA8702-15-D-0002. ZK gratefully acknowledges the NSF
(FAI 2040929 and IIS2211955), UPMC, Highmark Health, Abridge, Ford Research, Mozilla, the
PwC Center, Amazon AI, JP Morgan Chase, the Block Center, the Center for Machine Learning
and Health, and the CMU Software Engineering Institute (SEI) via Department of Defense contract
FA8702-15-D-0002, for their generous support of ACMI Lab’s research. JL is supported in part by NSF awards DMS-2309378 and IIS-2403275. AM is supported in part by a Microsoft Trustworthy AI Grant, an ONR grant and a David and Lucile Packard Fellowship. AR is supported in part by NSF awards IIS-2211907, CCF-2238523, IIS-2403275, an Amazon Research Award, a Google Research Scholar Award, and an OpenAI Superalignment Fast Grant.
\clearpage

\bibliographystyle{iclr2025_conference}
\bibliography{bibliography}

\clearpage
\appendix
\section*{Appendix}
\section{Omitted Proofs from \cref{section:map-separation}}\label{app:map-separation}

In this section we give omitted proofs and lemmas from \cref{section:map-separation}. 

\begin{figure}[t]
    \centering
    \includegraphics[width=0.5\textwidth]{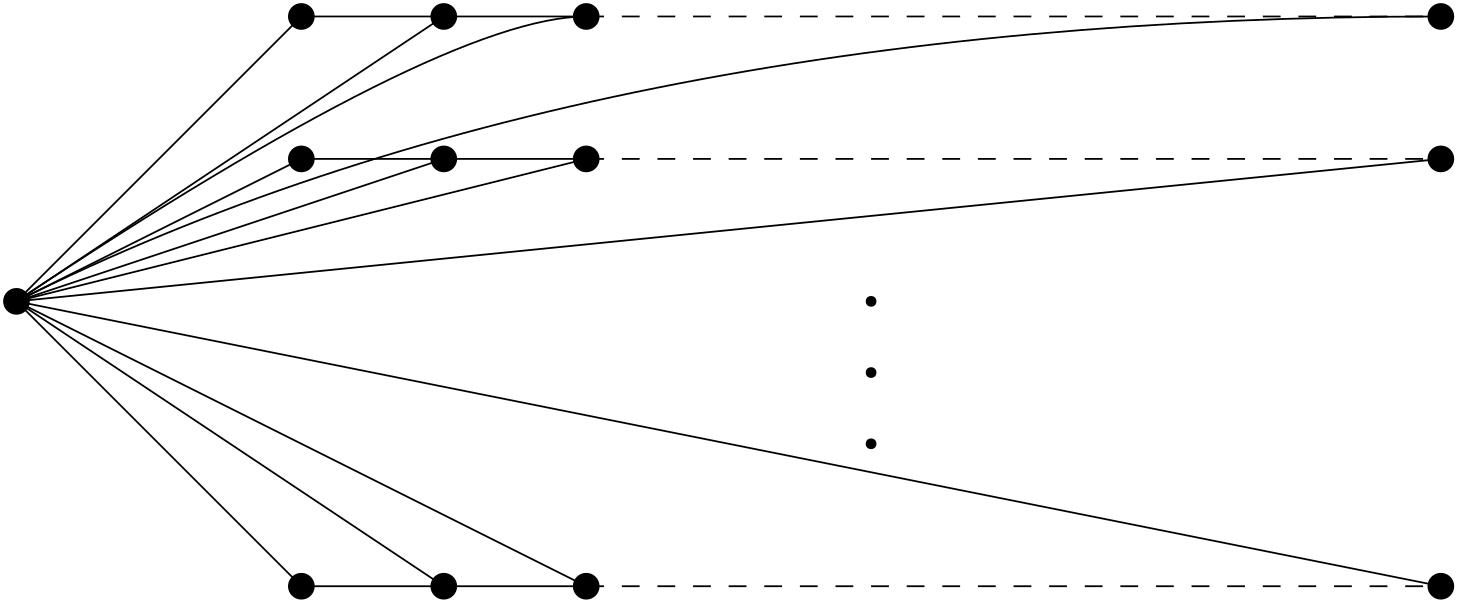}
    \caption{The graph $G$ for which \cref{thm:hub-separation} exhibits a separation between edge message-passing and node message-passing. The graph consists of $\sqrt{n}$ paths of length $\sqrt{n}$, as well as a single ``hub vertex'' connected to all other vertices.}
    \label{fig:path-graph}
\end{figure}

\begin{lemma}\label{lemma:evaluate-g}
Fix $n \in \NN$. Let $G$, $\Phi$ be as defined in \cref{thm:hub-separation}. Then there is an $O(n)$-time algorithm that computes a MAP evaluator for $G$ with potential function class $\Phi$.
\end{lemma}

\begin{proof}
Fix any $J \in \Phi^E$. As preliminary notation, for each $c,c_0 \in \{0,1\}$ and $i,j \in \sqrt{n}$, let $V(i,j) := \{0\}\cup\{(k,j): 1 \leq k \leq i\}$, and let $E(i,j)$ be the edge set of the induced subgraph $G[V(i,j)]$. Let
\begin{align*}
\hat{x}_{i,j}(c,c_0;J) &:= \argmin_{\substack{x \in \{0,1\}^{V(i,j)}: \\ x_0 = c_0 \,\land\, x_{(i,j)} = c}} \sum_{(a,b)\in E(i,j)} J_{\{a,b\}}(x_a,x_b), \\ 
\hat{C}_{i,j}(c,c_0;J) &:= \min_{\substack{x \in \{0,1\}^{V(i,j)}: \\ x_0 = c_0 \,\land\, x_{(i,j)} = c}} \sum_{(a,b)\in E(i,j)} J_{\{a,b\}}(x_a,x_b).
\end{align*}
For each $j \in [\sqrt{n}]$, let 
\[\hat x_j(c_0;J) := \hat x_{\sqrt{n},j}\left(\left(\argmin_{c \in \{0,1\}} \hat C_{\sqrt{n},j}(c,c_0;J)\right), c_0; J\right).\]
Finally, let $\hat{x}(c_0;J) \in \{0,1\}^V$ be the vector which takes value $c_0$ on vertex $0$, and value $\hat{x}_j(c_0;J)_i$ on vertex $(i,j)$ for all $i,j \in \sqrt{n}$. Let
\[\hat{x}(J) := \argmax_{c_0 \in \{0,1\}} p_J(\hat{x}(c_0;J)).\]
We claim that $\hat{x}(J)$ is a maximizer of $p_J(x)$. Indeed, for any fixed $c_0 \in \{0,1\}$, $\hat x(c_0;J)$ is a maximizer of $p_J(x)$ subject to $x_0 = c_0$, because under this constraint the maximization problem decomposes into $\sqrt{n}$ independent maximization problems, one for each path in $G$, which by definition are solved by $\hat x_1(c_0;J), \dots, \hat x_{\sqrt{n}}(c_0;J)$.

Moreover, it's straightforward to see that for any fixed $j$, $\hat C_j(c_0;J)$ can be computed in $O(\sqrt{n})$ time by dynamic programming. Indeed for any $i,j$,  $\hat C_{i,j}(c,c_0;J)$ can be computed in $O(1)$ time from $\hat C_{i-1,j}(0,c_0;J)$ and $\hat C_{i-1,j}(1,c_0;J)$ as well as $J_{\{0,(i,j)\}}$ and $J_{\{(i-1,j),(i,j)\}}$. Once the values $\hat C_{i,j}(c,c_0;J)$ have been computed for all $i \in [\sqrt{n}]$ and $c \in \{0,1\}$, the vector $\hat x_j(c_0;J)$ can be computed in $O(\sqrt{n})$ time via a reverse scan over $i = \sqrt{n},\dots, 1$. It follows that $\hat x(J)$ can be computed in $O(n)$ time.
\end{proof}

\begin{proof}[Proof of \cref{prop:node-edge-simulation}]
We claim that there is a node message-passing protocol $P'$ on $G$ with $T+1$ rounds that at each time $t \in [T+1]$ has computed
\[P'_t(v;I) = (P_{t-1}(e;I))_{e \in M_G(v)}.\]
We argue inductively. Since $P_0 \equiv 0$, it's clear that this can be achieved for $t = 1$. Fix any $t > 1$ and suppose that $P'_{t-1}(u;I) = (P_{t-2}(e;I))_{e \in M_G(u)}$ for all $u \in V$ and inputs $I$. For each $v \in V$, we define a function $f'_{t,v}$ by
\[f'_{t,v}((c(v'))_{v' \in N_G(v)}, (I(e))_{e \in M_G(v)})_{e^\st} := f_{t-1,e^\st}((c(v)_e)_{e \in M_G(v)}, (c(v^\st)_e)_{e \in M_G(v^\st)}, I(e^\st))\]
for each $e^\st = (v,v^\st) \in M_G(v)$. Then by definition and the inductive hypothesis, we have
\begin{align*}
P'_t(v;I)_{e^\st}
&= f'_{t,v}((P'_{t-1}(v';I))_{v'\in N_G(v)}, (I(e))_{e \in M_G(v)})_{e^\st} \\ 
&= f_{t-1,e^\st}((P'_{t-1}(v;I)_e)_{e \in M_G(v)}, (P'_{t-1}(v^\st;I)_e)_{e \in M_G(v^\st)}, I(e^\st)) \\ 
&= f_{t-1,e^\st}((P_{t-2}(e;I))_{e \in M_G(v)}, (P_{t-2}(e;I)_e)_{e \in M_G(v^\st)}, I(e^\st)) \\ 
&= P_{t-1}(e^\st; I)
\end{align*}
for any edge $e^\st = (v,v^\st) \in E$, since $M_G(e) = M_G(v) \cup M_G(v^\st)$. This completes the induction and shows that $P'_{T+1}(v;I) = (P_T(e;I))_{e \in M_G(v)}$ for all $v, I$. Replacing $f'_{T+1,v}$ by $\tilde f_{T,v} \circ f'_{T+1,v}$ completes the proof.
\end{proof}

\section{Omitted Proofs from \cref{section:mapsymmetry}}\label{app:mapsymmetry}

\begin{proof}[Proof of \cref{t:symmetryalone}]
Without loss of generality, we may assume that the functions $(\fsym_t)_{t \in [T]}$ and $\tfsym$ are all the identity function (on the appropriate domains). The reason is that any symmetric edge message-passing protocol $\tilde{P}$ on $T$ rounds may be simulated by running $P$ and then applying a universal function (depending only on $\tilde{P}$) to each node's output value \--- see \cref{lemma:univ-sym-protocol}. 

We argue by induction that for each $t \in [T]$, there is a $(t+1)$-round symmetric node message-passing protocol that, on any input $I$, computes the function $Q_t(u;I) := \multiset{P_t(e; I): e \in M_G(u)}$ for every node $u \in V$. Consider $t = 1$. For any $e = (u,v) \in E$, we have by symmetry and the initial assumption that
\[P_1(e;I) = (I(e), 0, \multiset{\multiset{0: v' \in N_G(u)}, \multiset{0: u' \in N_G(v)}}).\]
We define a two-round node message-passing protocol on $G$ where the first update at node $u$ computes
\[P'_1(u;I) = \multiset{I(\{u,v\}): v \in N_G(u)}\]
and the second update at node $u$ computes
\begin{align*}
(P'_1(u;I), \multiset{(P'_1(v;I), I(\{u,v\})): v \in N_G(u)})
&\mapsto \multiset{(I(\{u,v\}),0, |N_G(u)|, |P'_1(v;I)|): v \in N_G(u)} \\ 
&\mapsto \multiset{(I(\{u,v\}),0,\multiset{|N_G(u)|, |P'_1(v;I)|}): v \in N_G(u)} \\
&= \multiset{P_1(\{u,v\};I): v \in N_G(u)} =: P'_2(u;I)
\end{align*}
since $|P'_1(v;I)| = |N_G(v)|$. By construction, this protocol is symmetric, which proves the induction for step $t=1$. 

Now pick any $t > 1$. For any $e = \{u,v\} \in E$, we have
\begin{align*}
P_t(e;I) 
&= (I(e),P_{t-1}(e;I),\multiset{Q_{t-1}(u;I),Q_{t-1}(v;I)})
\end{align*}
By the induction hypothesis, there is a $t$-round symmetric node message-passing protocol $P'$ that, at node $v$ on input $I$, computes
\[P'_t(v;I) = \multiset{P_{t-1}(\{v,v'\};I): v' \in N_G(v)} = Q_{t-1}(v;I).\]
Note that since $P_{t-1}(e;I)$ is an element of the tuple $P_t(e;I)$, for each $1 \leq s \leq t-1$ there is a fixed function $\gamma_s$ such that $\gamma_s(Q_{t-1}(v;I)) = Q_s(v;I)$ for all $v, I$. Using this fact, we extend $P'$ to $t+1$ rounds, defining the update at round $t+1$ and node $u$ as follows:
\begin{align*}
&(P'_t(u;I),\multiset{(P'_t(v;I), I(\{u,v\})): v \in N_G(u)}) \\
&= (Q_{t-1}(u;I), \multiset{(Q_{t-1}(v;I), I(\{u,v\})): v \in N_G(u)}) \\ 
&\mapsto (Q_{1:t-1}(u;I), \multiset{(Q_{1:t-1}(v;I), I(\{u,v\})): v \in N_G(u)}) \\ 
&\mapsto (Q_{1:t-1}(u;I), \multiset{(Q_{1:t-1}(v;I), I(\{u,v\})): v \in N_G(u)}) \\
&= \multiset{(I(\{u,v\}),\multiset{Q_{1:t-1}(u;I), Q_{1:t-1}(v;I)}): v \in N_G(u)} \\ 
&\mapsto \multiset{(I(\{u,v\}),P_{t-1}(\{u,v\};I), \multiset{Q_{t-1}(u;I), Q_{t-1}(v;I)}): v \in N_G(u)} =: P'_{t+1}(u;I)
\end{align*}
where $Q_{1:t-1}(u;I)$ refers to the tuple $(Q_1(u;I),\dots,Q_{t-1}(u;I))$. The first map is well-defined due to the existence of the functions $\gamma_1,\dots,\gamma_{t-1}$, and the final map is well-defined because the definition of $P_{t-1}(\{u,v\};I)$ can be iteratively unpacked, and it is ultimately a function of \[(I(\{u,v\}),\multiset{Q_{1:t-1}(u;I), Q_{1:t-1}(v;I)}).\] This shows that $P'$ computes $Q_t(v;I)$ at node $u$ on input $I$. By construction, $P'$ is symmetric. This completes the induction. Since $Q_T(u;I)$ is precisely the output of $P$ at node $u$ on input $I$ (after the node aggregation step), this shows that $P$ can be simulated by a $(T+1)$-round symmetric node message-passing protocol on $G$. 
\end{proof}

\begin{lemma}\label{lemma:univ-sym-protocol}
Let $T \geq 1$, and let $P = ((f_{t,e})_{t \in [T], e \in E}, (\tilde f_v)_{v \in V})$ be a symmetric edge message-passing protocol on $G = (V,E)$ with $T$ rounds. Consider the $T$-round edge message-passing protocol $P^\circ = ((f^\circ_{t,e})_{t \in [T], e \in E}, (\tilde f^\circ_v)_{v \in V})$ where for all $t,e$,
\[f_{t,e}^\circ((c(e'))_{e' \in M_G(e)}, I(e)) := (I(e), c(e),\multiset{c(\{u,v'\}): v' \in N_G(u)}, \multiset{c(\{u',v\}): u' \in N_G(v)}),\]
and for every $v \in V$, 
\[\tilde f_v^\circ((c(e))_{e \in M_G(v)}) := \multiset{c(e): e \in M_G(v)}.\]
Then there is a function $h$ such that $\tilde f_v((P_T(e;I))_{e \in M_G(v)}) = h(\tilde f^\circ_v((P^\circ_T(e;I))_{e \in M_G(v)}))$ for all $v, I$.
\end{lemma}

\begin{proof}
We prove by induction that for each $t \in \{0,\dots,T\}$ there is a function $h_t$ such that $P_t(e;I) = h_t(P^\circ_t(e;I))$ for all $e,I$. For $t = 0$ this is immediate from the convention that $P_0 \equiv P^\circ_0 \equiv 0$. Fix any $t \in \{1,\dots,T\}$. Since $P$ is symmetric, there is a function $\fsym_t$ so that for all $e = (u,v) \in E$ and inputs $I$,
\begin{align*}
P_t(e;I)
&= \fsym_t(I(e),P_{t-1}(e;I),\multiset{P_{t-1}(\{u,v'\};I): v' \sim u}, \multiset{P_{t-1}(\{u',v\};I): u' \sim v}) \\ 
&= \fsym_t(I(e),h_{t-1}(P^\circ_{t-1}(e;I)),\multiset{h_{t-1}(P^\circ_{t-1}(\{u,v'\};I)): v' \sim u}, \multiset{h_{t-1}(P^\circ_{t-1}(\{u',v\};I)): u' \sim v})
\end{align*}
which is indeed a well-defined function (independent of $e,I$) of 
\[P^\circ_t(e;I) = (I(e),P^\circ_{t-1}(e;I),\multiset{P^\circ_{t-1}(\{u,v'\};I): v' \sim u}, \multiset{P^\circ_{t-1}(\{u',v\};I): u' \sim v}).\]
This completes the induction. Finally, since $P$ is symmetric, there is a function $\tfsym$ such that $\tilde f_v((P_T(e;I))_{e\in M_G(v)}) = \tfsym(\multiset{P_T(e;I): e \in M_G(v)})$ for all $v,I$. Hence we can write
\begin{align*}
\tilde f_v((P_T(e;I))_{e\in M_G(v)}) 
&= \tfsym(\multiset{P_T(e;I): e \in M_G(v)}) \\
&= \tfsym(\multiset{h_T(P^\circ_T(e;I)): e \in M_G(v)})
\end{align*}
which is a well-defined function (independent of $v,I$) of $\multiset{P^\circ_T(e;I): e \in M_G(v)}$ as needed.
\end{proof}

\section{A quantitatively tight depth/memory separation}\label{app:cc}

For each $n \in \NN$, let $K_n := ([n], E_n)$ be the complete graph on $[n]$. In this section we show that there is a function that can be computed by an edge message-passing protocol on $K_n$ with constant rounds and constant memory per processor, but for which any node message-passing protocol with $T$ rounds and $B$ bits of memory requires $TB \geq \Omega(n)$. We remark that this separation is quantitatively tight due to \cref{prop:node-edge-simulation}, although it is possible that a larger (e.g. even super-polynomial in $n$) depth separation may be possible if the node message-passing protocol is restricted to constant memory per processor.

At a technical level, the lower bound proceeds via a reduction from the \emph{set disjointness problem} in communication complexity, similar to the lower bounds in \cite{loukas2019graph}.

\begin{definition}
Fix $m \in \NN$. The set disjointness function $\disj_m: \{0,1\}^m \times \{0,1\}^m \to \{0,1\}$ is defined as
\[\disj_m(A, B) := \mathbbm{1}[\forall i \in [m]: A_iB_i = 0].\]
\end{definition}

The following fact is well-known; see e.g. discussion in \cite{haastad2007randomized}.

\begin{lemma}\label{lemma:set-disj}
In the two-party deterministic communication model, the deterministic communication complexity of $\disj_m$ is at least $m$.
\end{lemma}

The main result of this section is the following:

\begin{theorem}
Fix any even $n \in \NN$. Define $g: \{0,1\}^{E_n} \to \{0,1\}^n$ by 
\[g(I)_v := \mathbbm{1}[\exists \{i,j\} \in E_n: i,j \leq n/2 \land I(\{i,j\}) = I(\{n+1-i,n+1-j\}) = 1]\]
for all $I \in \{0,1\}^{E_n}$ and $v \in [n]$. Then the following properties hold:
\begin{itemize}
\item Any node message-passing protocol on $K_n$ with $T$ rounds and $B$ bits of memory that computes $g$ requires $TB \geq \Omega(n)$
\item There is an edge message-passing protocol on $K_n$ with $O(1)$ rounds and $O(1)$ bits of memory that computes $g$.
\end{itemize}
\end{theorem}

\begin{proof}
Let $m := \binom{n/2}{2}$. Let $P = (f_{t,v})_{t,v}$ be a node message-passing protocol on $K_n$ that computes $g$ with $T$ rounds and $B$ bits of memory. We design a two-party communication protocol for $\disj_m$ as follows. Suppose that Alice holds input $X \in \{0,1\}^m$ and Bob holds input $Y \in \{0,1\}^m$. Let us index the edges $\{i,j\} \in E_n$ with $i,j \leq n/2$ by $[m]$, and similarly index the edges $\{i,j\} \in E_n$ with $i,j > n/2$ by $[m]$, in such a way that edge $\{i,j\}$ has the same index as edge $\{n+1-i,n+1-j\}$. Let $I \in \{0,1\}^{E_n}$ be defined by 
\[I(\{i,j\}) := \begin{cases} X_{\{i,j\}} & \text{ if } i,j \leq n/2 \\ Y_{\{i,j\}} & \text{ if } i,j > n/2 \\ 0 & \text{ otherwise } \end{cases}.\]
Initially, Alice computes $\hat P_0(v) := 0$ for all $v \in \{1,\dots,n/2\}$, and Bob computes $\hat P_0(v) := 0$ for all $v \in \{n/2+1,\dots,n\}$. The communication protocol then proceeds in $T$ rounds. At round $t \in [T]$, Alice sends $(\hat P_{t-1}(v))_{1 \leq v \leq n/2}$ to Bob, and Bob sends $(\hat P_{t-1}(v))_{n/2 + 1 \leq v \leq n}$ to Alice. Alice then computes
\[\hat P_t(v) := f_{t,v}((\hat P_{t-1}(v'))_{v' \in [n]}, (I(e))_{e \in M_{K_n}(v)})\]
for each $1 \leq v \leq n/2$, and Bob computes the same for each $n/2 < v \leq n$. Note that for any $i \leq n/2$ and edge $e \in M_{K_n}(i)$, Alice can compute $I(e)$. Similarly, for any $i > n/2$ and edge $e \in M_{K_n}(i)$, Bob can compute $I(e)$. Thus, this computation is well-defined. After round $T$, Alice and Bob output $1-\hat P_T(1)$ and $1-\hat P_T(n)$ respectively.

This defines a communication protocol. Since $\hat P_t(v) \in \{0,1\}^B$ for each $v \in [n]$ and $t \in [T]$, the total number of bits communicated is at most $nBT$. Moreover, by induction it's clear that Alice and Bob output $1-P_T(1;I)$ and $1-P_T(n;I)$ respectively. By assumption that $P$ computes $g$ and the fact that $g(I)_v = 1-\disj_m(X,Y)$ for all $v \in [n]$, we have that $1-P_T(1;I) = 1-P_T(n;I) = 0$ if $\disj_m(I) = 0$, and $1-P_T(1;I) = 1-P_T(n;I) = 1$ if $\disj_m(I) = 1$. Thus, this communication protocol computes $\disj_m$. By \cref{lemma:set-disj}, it follows that $nBT \geq m = \Omega(n^2)$, so $BT = \Omega(n)$ as claimed.

Next, we exhibit an edge message-passing protocol on $K_n$ that computes $g$ with six rounds and one bit of memory. For $1 \leq t \leq 6$ and $e \in E_n$, define $f_{t,e}: \{0,1\}^{M_G(e)} \times \{0,1\} \to \{0,1\}$ as follows:
\begin{align*}
f_{1,\{i,j\}}(x, y) &:= y \\ 
f_{2,\{i,j\}}(x, y) &:= x_{\{n+1-i,j\}} \\ 
f_{3,\{i,j\}}(x, y) &:= x_{\{i,n+1-j\}} \\ 
f_{4,\{i,j\}}(x, y) &:= \mathbbm{1}[y = x_{\{i,j\}} \land i,j \leq n/2] \\ 
f_{5,\{i,j\}}(x, y) &:= \mathbbm{1}[\exists k \in [n]: x_{\{i,k\}} = 1] \\ 
f_{6,\{i,j\}}(x, y) &:= \mathbbm{1}[\exists k \in [n]: x_{\{i,k\}} = 1].
\end{align*}
Also define $\tilde f_v: \{0,1\}^{M_G(v)} \to \{0,1\}$ for each $v \in [n]$ by $\tilde f_v(x) := x_{\{x,1\}}$. It can be checked that the computation of $P$ at timestep $t=6$ is 
\[P_6(\{i,j\};I) := \mathbbm{1}[\exists k,\ell \in [n/2]: I(\{k,\ell\}) = I(\{n+1-k,n+1-\ell\})] = g(I).\]
From the definition of $\tilde f$, it follows that $P$ computes $g$.
\end{proof}

\section{Further details on synthetic task over Ising models} 
\label{a:isingdetails}

\subsection{Background on belief propagation}
A classical way to calculate the marginals $\{\mathbb{E}[x_i]\}$ of an Ising model, when the associated graph is a tree, is to iterate the message passing algorithm: 
\begin{equation}
    \nu^{(t+1)}_{i \to j} = \tanh \left(
    h_i + 
    \sum_{k \in \partial_i \backslash j}
    \tanh^{-1} \left(\tanh(J_{ik} )\nu_{k \to i}^{(t)}\right)
    \right)
    \label{eq:bpupdates}
\end{equation} 
When the graph is a tree, it is a classical result (\citep{mezard2009information}, Theorem 14.1) that the above message-passing algorithm converge to values $\nu^*$ that yield the correct marginals, namely: \[\mathbb{E}[x_i] = \tanh\left(h_i + \sum_{k \in \partial_i}
    \tanh^{-1}\left(\tanh(J_{ik}) \nu^*_{k\to i}\right)\right).\]

The reason the updates converge to the correct values on a tree topology is that they implicitly simulate a dynamic program. Namely, we can write down a recursive formula for the marginal of node $i$ which depends on sums spanning each of the subtrees of the neighbors of $i$ (i.e., for each neighbor $j$, the subgraph containing $j$ that we would  get if we removed edge $\{i,j\}$). 

If we root the tree at an arbitrary node $r$, we can see that after completing a round of message passing from the leaves to the root, and another from the root to the leaves, each subtree of $i$ will be (inductively) calculated correctly.     

Moreover, even though the updates \eqref{eq:bpupdates} are written over edges, the dynamic programming view makes it clear an equivalent message-passing scheme can be written down where states are maintained over the \emph{nodes} in the graph. Namely, for each node $v$, we can maintain two values $h_{v,\mbox{down}}$ and $h_{v,\mbox{up}}$, which correspond to the values that will be used when $v$ sends a message upwards (towards the root) or downwards (away from the root). Then, for appropriately defined functions $F,G$ (depending on the potentials $J$ and $h$), one can ``simulate'' the updates in \eqref{eq:bpupdates}: \begin{align}
    h^{(t+1)}_{v,\mbox{up}} &\leftarrow F\left(\{h^{(t)}_{w, \mbox{up}}: w \in v \cup \mbox{Children}(v)\}\right) \label{eq:nodedir1}\\ 
    h^{(t+1)}_{v, \mbox{down}} &\leftarrow G\left(h^{(t)}_{\mbox{Parent}(v), \mbox{down}}, \left\{h^{(t)}_{w, \mbox{up}}\right\}_{w \in \mbox{Children}(v)}\right) \label{eq:nodedir2}
\end{align}       

Intuitively, $h_{v,\mbox{up}}$ captures the effective external field induced by the subtree rooted at $v$ on $\mbox{Parent}(v)$. After the upward messages propagate, the root $r$ can compute its correct marginal. Once $h_{\mbox{Parent}(v),\mbox{down}}$ is the correct marginal for $\mbox{Parent}(v)$ at some step, $h_{v,\mbox{down}}$ will be the correct marginal for $v$ at all subsequent steps.

\subsection{GCN-based architectures to calculate marginals}   
\label{s:archsising}
The belief-propagation updates \eqref{eq:bpupdates} naturally fit the general edge-message passing paradigm from \eqref{eq:edgempnn}. In fact, they fit even more closely a ``directed'' version of the paradigm, in which each edge $\{i,j\}$ maintains two embeddings $h_{i \to j}, h_{j \to i}$, such that the embedding for direction $h_{i \to j}$ depends on the embeddings $\{h_{k \to i}\}_{\{k,i\} \in E}$. With this modification to the standard edge GCN architecture \cref{eq:edgegcn}, it is straightforward to implement \eqref{eq:bpupdates} with one layer, using a particular choice of activation functions and weight matrices $W$ (since, in particular, in our dataset all edge potentials $J_{i,j}$ are set to 1). 
Similarly, with a directed version of the node GCN architecture \cref{eq:nodegcn}, where each node maintains an ``up'' embedding as well as a ``down'' embedding, it is straightforward to implement the ``node-based'' dynamic programming solution \eqref{eq:nodedir1}-\eqref{eq:nodedir2}.

We call the architectures that do not maintain directionality Node-U and Edge-U (depending on whether they use a node-based or edge-based GCN). We call the ``directed'' architectures Node-D and Edge-D respectively. Since there are only initial node features (input as node potentials $\{h_i\}_{i \in }$), for the edge based architectures we initialize the edge features as a concatenation of the node features of the endpoints of the edge. The results we report for each architecture are the best over a sweep of  
depth $\in \{5, 10, 15, 20, 25, 30\}$ and width  $\in \{10, 32, 64\}$. 
 
\subsection{Edge-based models improve over node-based models} 

In Figure~\ref{fig:isingexps} we show the results for several tree topologies: a complete binary tree (of size 31), a path graph (of size 30), and uniformly randomly chosen trees of size 30 (the results in Figure~\ref{fig:isingexps} are averaged over 3 samples of tree). The architectures in the legend (Node-U, Edge-U, Node-D, Edge-D) are based on a standard GCN, and detailed in Section~\ref{s:archsising}

We can see that for both the undirected and directed versions, adding edge embeddings improves performance. The improved performance of all directed versions compared to their undirected counterpart is not very surprising: the standard, undirected GCN architecture treats all neighbors symmetrically --- hence, the directed versions can more easily simulate something akin to the belief propagation updates \eqref{eq:bpupdates} as well as the node-based dynamic programming \eqref{eq:nodedir1}-\eqref{eq:nodedir2}.


\begin{figure*}[t]
    \centering
    \includegraphics[width=\linewidth]{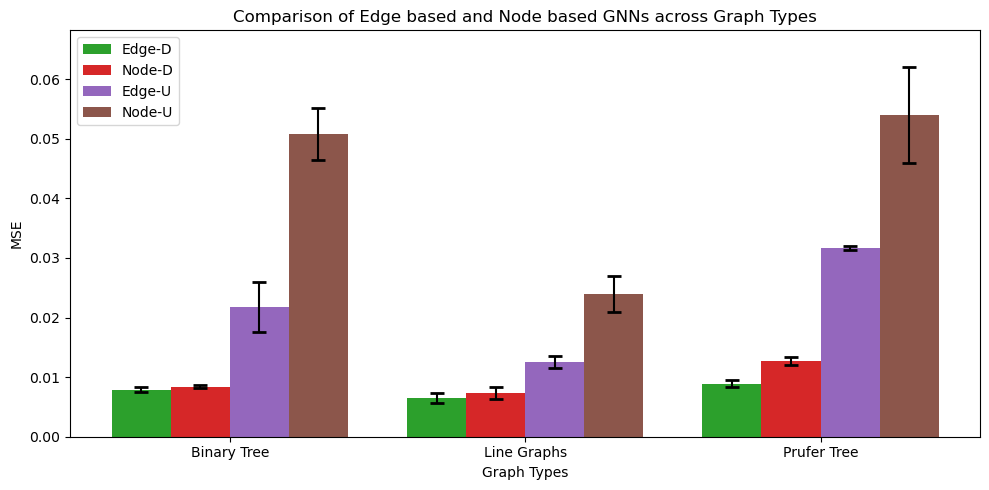}
    \caption{Comparison of four architectures for calculating node marginals in an Ising model. The architectures considered are node-embedding \eqref{eq:nodegcn} and edge-embedding \eqref{eq:edgegcn} versions of a GCN (correspondingly labeled Node-U and Edge-U), as well as their ``directed'' counterparts, as described in Section~\ref{s:archsising}, correspondingly labeled Node-D and Edge-D. The x-axis groups results according to the topology of the graph, the y-axis is MSE (lower is better). 
    The mean and variances are reported over $3$ runs for the best choice of depth and width over the sweep described in Section~\ref{s:archsising}.} 
    \label{fig:isingexps}
\end{figure*}

\newpage

\end{document}